%% file: main.tex
\documentclass[twoside]{article}

\usepackage[accepted]{aistats2024}
\usepackage[round]{natbib}

\usepackage[utf8]{inputenc} 
\usepackage[T1]{fontenc}    
\usepackage{hyperref}       
\usepackage{url}            
\usepackage{booktabs}       
\usepackage{amsfonts}       
\usepackage{nicefrac}       
\usepackage{microtype}      
\usepackage{cancel}
\usepackage{autonum}

\usepackage{graphicx}
\usepackage{amsmath}
\usepackage{subfigure}
\usepackage{subcaption}
\usepackage{bigstrut}
\usepackage{paralist}
\usepackage{multicol}
\usepackage{multirow}
\usepackage{tabularx}
\usepackage{float}
\usepackage[table]{xcolor}
\usepackage{colortbl}
\usepackage{tikz}
\usepackage{tikz-3dplot}
\usetikzlibrary{arrows}
\usepackage{bm}

\definecolor{mydarkblue}{rgb}{0,0.08,0.45}
\hypersetup{ %
pdftitle={},
pdfkeywords={},
pdfborder=0 0 0,
pdfpagemode=UseNone,
colorlinks=true,
linkcolor=mydarkblue,
citecolor=mydarkblue,
filecolor=mydarkblue,
urlcolor=mydarkblue,
}

\usepackage{cuted}
\usepackage{flushend}
\usepackage{balance}
\usepackage{enumitem}
\usepackage{wrapfig}
\usepackage{pifont}

\usepackage{bm}
\usepackage{makecell}
\usepackage{dsfont}
\usepackage{amsthm}

\usepackage{graphicx}
\newcommand{\ours}{\textsf{CCEM}}

\newcommand\op[1]{\operatorname{#1}}
\newcommand{\eg}{\textit{e.g., }}
\newcommand{\ie}{\textit{i.e., }}

\makeatletter
\def\@fnsymbol#1{\ensuremath{\ifcase#1\or * \or w \or c \or m \or p \or \dagger\or \mathsection\or
   \ddager\or \mathparagraph\or \|\or **\or \dagger\dagger
   \or \ddagger\ddagger \else\@ctrerr\fi}}
\makeatother

\usepackage{caption}

\newtheorem{theorem}{Theorem}
\newtheorem{proposition}{Proposition}
\newtheorem{lemma}{Lemma}
\newtheorem{corollary}{Corollary}
\newtheorem{definition}{Definition}

\newtheorem{remark}{Remark}
\newtheorem{example}{Example}

\input{math_commands.tex}

\begin{document}

%

%

\twocolumn[

\aistatstitle{Analysis of Using Sigmoid Loss for Contrastive Learning}

\aistatsauthor{ Chungpa Lee \And Joonhwan Chang \And Jy-yong Sohn }

\aistatsaddress{ Yonsei University \And  Yonsei University \And  Yonsei University } ]

\begin{abstract}
Contrastive learning has emerged as a prominent branch of self-supervised learning for several years. Especially, CLIP, which applies contrastive learning to large sets of captioned images, has garnered significant attention.
Recently, SigLIP, a variant of CLIP, has been proposed, which uses the sigmoid loss instead of the standard InfoNCE loss. SigLIP achieves the performance comparable to CLIP in a more efficient manner by eliminating the need for a global view. However, theoretical understanding of using the sigmoid loss in contrastive learning is underexplored. 
In this paper, we provide a theoretical analysis of using the sigmoid loss in contrastive learning, in the perspective of the geometric structure of learned embeddings. 
First, we propose the \textit{double-Constant Embedding Model} (\ours{}), a framework for parameterizing various well-known embedding structures by a single variable. Interestingly, the proposed \ours{} 
is proven to contain the optimal embedding with respect to the sigmoid loss. Second, we mathematically analyze the optimal embedding minimizing the sigmoid loss for contrastive learning. The optimal embedding ranges from simplex equiangular-tight-frame to antipodal structure, depending on the temperature parameter used in the sigmoid loss.
Third, our experimental results on synthetic datasets coincide with the theoretical results on the optimal embedding structures.
\end{abstract}

\section{Introduction}\label{sec:intro}
Contrastive learning (CL), which learns the embedding space based on the comparison between training data, is recently considered as one of the popular method for learning good representations \citep{liu2021self, jaiswal2020survey, chen2020simple, tian2020makes}. The basic idea is, to train an encoder such that similar training data (called positive pairs) are mapped to closer embedding vectors, while the embedding vectors for dissimilar training data (called negative pairs) are far apart. For example, in~\cite{chen2020simple}, different augmentations of the same image are labeled as a positive pair, while augmentations of different images are labeled as a negative pair. In~\cite{radford2021clip}, the caption text and the image extracted from a single captioned image are considered as a positive pair, while those extracted from different captioned images are considered as a negative pair. 
In general, different \textit{views}  (augmentations or modalities) on a single entity form a positive pair, while views on different entities form a negative pair.

Interestingly, self-supervised learning trained by such comparison-based criterion alone, without any labeled data, is sufficient to learn meaningful features 
~\citep{he2020momentum,chen2020simple,radford2021clip,jia2021scaling,girdhar2023imagebind}, recently outperforming the features learned by supervised learning.
Various ways of defining contrastive loss have been considered in the past years~\citep{schroff2015facenet,oh2016deep,sohn2016improved,gutmann2010noise, gutmann2012noise, hyvarinen2019nonlinear}, see Appendix~\ref{sec:compare:loss} for more details. Especially, one of the pioneering works on multi-modal CL called CLIP~\citep{radford2021clip} uses 
the InfoNCE loss~\citep{oord2018representation}, which is theoretically analyzed in recent papers~\citep{saunshi2019theoretical, saunshi2022understanding, wang2020understanding, wang2021understanding, chen2021intriguing, arora2019theoretical}. 
Above all,~\citet{lu2022neural} shows that the optimal solution of the variational problem of minimizing cross-entropy loss, which coincides with the InfoNCE loss, is structured by the simplex equiangular-tight-frame (ETF) if feature vectors are in a sufficiently high-dimensional space.

Recently, a variant of CLIP called SigLIP is proposed~\citep{zhai2023sigmoid}, which uses the sigmoid loss instead of the InfoNCE loss, motivated by the fact that computing with the sigmoid loss is much more efficient than using the InfoNCE loss.
More precisely, CL using the InfoNCE loss reveals remarkable performance only if the negative sample size, which is limited by its mini-batch size, is sufficiently large \citep{chen2020simple, zhang2022dual}. However, the necessity for a larger batch size imposes the computation burdens on CL using the InfoNCE loss. On the other hand, learning CL with SigLIP is independent of batch size, enabling comparable performance with significantly reduced computational demands.
Unfortunately, there is no theoretical discussions on the behavior of SigLIP.

\paragraph{Our Contributions}

In this paper, we have taken the first step towards understanding SigLIP, or in general, using the sigmoid loss for contrastive learning.
Specifically, we have the following contributions, in terms of analyzing the embeddings learned by the sigmoid loss.
\begin{itemize}
    \item We define the \emph{double-Constant Embedding Model} (\ours{}), a framework that models various structures formed by multiple embedding vectors. 
    This model contains well-known structures including the simplex ETF structure and the antipodal structure. 
    \item We show that optimal embedding which minimizes the sigmoid loss can be found in the proposed \ours{}. In other words, it is sufficient to find the optimal solution within the structures that can be represented in the \ours{} framework.
    \item The optimal embedding structure that minimizes sigmoid loss is specified. The optimal embedding forms simplex ETF when the temperature parameter is sufficiently large, while it forms an antipodal structure when the temperature is smaller than a threshold.
    \item Our experimental results on synthetic datasets support our theoretical findings on the behavior of optimal embeddings.
\end{itemize}

\section{Related Works}

\vspace{-3mm}
\paragraph{Contrastive Learning}

Self-supervised learning has gained significant attention as it enables learning from large-scale data without the need for laborious supervision labeling \citep{goyal2019scaling, chen2020simple, he2020momentum, chen2020big}. Among various techniques, CL has been widely utilized in computer vision \citep{caron2020unsupervised, chen2021exploring, grill2020bootstrap} and natural language processing \citep{giorgi2020declutr, wu2020clear} domains due to the excellent performance of learned embeddings.
CL is a similarity-based training method, aiming to position similar data close to each other in the embedding space while pushing dissimilar data far apart. Moreover, it is known that CL can be applied in multi-modal scenarios~\citep{radford2021clip, jia2021scaling, ma2020active, elizalde2023clap,li2021supervision,mu2022slip,yao2021filip,goel2022cyclip,zhai2023sigmoid}, which leverage data with multiple views from different modalities such as images paired with text captions. 
In this paper, we provide analysis on the CL, in the perspective of the learned embedding structure, 
which can be applied for both uni-modal and multi-modal scenarios.

\paragraph{Analysis on Geometry of Embeddings}

Several researchers focused on the geometric structure of optimal embeddings learned by CL such as antipodal and simplex ETF 
\citep{lu2022neural, awasthi2022more, sreenivasan2023mini, graf2021dissecting}.
Especially, \cite{lu2022neural} shows that the optimal feature vector 
solution of a variational problem related to cross-entropy loss forms a simplex ETF, based on previous results on the ETFs~\citep{sustik2007existence, fickus2015tables, strohmer2003grassmannian}.
On this basis,~\citet{sreenivasan2023mini} demonstrated that optimal representations for minimizing the contrastive loss form a simplex ETF  when $d\geq N-1$. 
Our paper is similar to this work in terms of finding optimal embeddings for minimizing contrastive losses.
While existing works solve the optimization problem for softmax-based losses, our paper solves the problem for the sigmoid loss.
In addition, we define a novel embedding structure model which includes the simplex ETF and antipodal structure in a single framework, which helps find the embeddings minimizing the sigmoid loss.

\paragraph{Analysis on Temperature Parameter of CL}
Various theoretical studies have explored on the temperature parameter of CL \citep{zhang2021temperature, zhang2022dual, zhang2022does}.
Especially, \cite{wang2021understanding} shows that the temperature parameter of CL controls the strength of penalties on the hard negative sample by using gradient analysis.
\cite{kukleva2023temperature} introduces a schedule for the temperature parameter of CL to improve the representation quality.
However, previous research focused on the role of the temperature parameter during the learning process without considering the optimal embeddings from different temperature parameters. Particularly, we demonstrated that the use of the sigmoid loss results in embeddings that are highly dependent on the temperature parameter.



\section{Problem Setup}\label{sec:problem setup} 

Here we first provide backgrounds on CL problem we are focusing on, and introduce several notations necessary for defining our problem setup.  Then, we provide the formal description of our target problem of optimizing the embedding with the sigmoid loss, and describe the skeleton of our paper. 

\subsection{Backgrounds}
CL uses $N$ paired instances $\{(\vx_i, \vy_i)\}_{i=1}^N$, where $\vx_i$ and $\vy_i$ for the same index $i$ are two different views or modalities of the same object. The goal of CL is to train encoders $f_x$ and $f_y$ in a way that the output of the encoders for different views ($\vx_i$ and $\vy_i$) of the same object are mapped 
into similar embedding vectors, \ie $f_x(\vx_i) \approx f_y(\vy_i)$.


Denote the collections of embedding vectors by $\mU=[\vu_1, \vu_2, \cdots ,\vu_N]$ and $\mV=[\vv_1, \vv_2, \cdots ,\vv_N]$ where $\vu_i = f_x(\vx_i) \in \mathbb{R}^d$ and $\vv_i = f_y(\vy_i) \in \mathbb{R}^d$ for $i\in [N]$. 
For mathematical simplicity, we use the normalized embedding space where $\vu_i^\top\vu_i=\vv_i^\top\vv_i=1$ for all $i \in [N]$, which is commonly used in related works~\citep{awasthi2022more,sreenivasan2023mini,huang2023towards}.

Here we focus on two popular contrastive losses: the InfoNCE loss 
\begin{align}
   \Ls^{\op{InfoNCE}}(\mU, \mV) 
&:= -\frac{1}{N} \sum_{i=1}^N \log \frac{\exp (\vu_i^\top \vv_i /t)}{\sum_{j=1}^N \exp (\vu_i^\top \vv_j /t)}  \\
&\quad - \frac{1}{N} \sum_{i=1}^N \log \frac{\exp(\vu_i^\top \vv_i /t)}{\sum_{j=1}^N \exp (\vu_j^\top \vv_i /t)}\label{loss:info}
\end{align} 
considered in~\cite{oord2018representation,radford2021clip}, 
and  
the sigmoid loss 
\begin{align}\label{loss:siglip}
\Ls^{\op{sig}}(\mU, \mV) 
&:= -\frac{1}{N} \sum_{i=1}^N  \log \frac{1}{1+\exp (-t\vu_i^\top \vv_i+b)} \\
&\quad - \frac{1}{N} \sum_{i=1}^N \sum_{j\in [N] \setminus \{i\}}  \log \frac{1}{1+\exp (t\vu_i^\top \vv_j-b)}
\end{align}

proposed by \cite{zhai2023sigmoid} as an alternative to the InfoNCE loss  $\Ls^{\op{InfoNCE}}$ in Eq.~\ref{loss:info} for the purpose of efficient computation.
Here, $t>0$ and $b \geq 0$ are temperature and bias, respectively.



Under this setting, one important question is about finding the optimal embedding vectors
\begin{align}
    (\mU^{\star}, \mV^{\star}) := \argmin_{\mU,\mV} \Ls(\mU, \mV)
\end{align}
learned by a pre-defined contrastive loss $\Ls$. A recent work~\citep{lu2022neural} shows that the optimal embedding vectors for the InfoNCE loss 
form a simplex ETF defined below:
\begin{definition} Suppose $d\geq N-1$. A set of $N$ vectors $\{\vw_i\}_{i=1}^N$ in $\sR^d$ is called 
$(N-1)$-simplex ETF, if 
\begin{align}
 \vw_i^\top\vw_i&=1, \quad \forall i \in [N], \\    
\vw_i^\top\vw_j&=-\frac{1}{N-1}, \quad \forall i, j \in [N] \text{ with } i \ne j.
\end{align}
\end{definition}



\begin{corollary}[Theorem 1 of~\cite{lu2022neural}]\label{thm:lu}
Let 
$    (\mU^{\star}, \mV^{\star}) = \argmin_{\mU, \mV} \Ls^{\op{InfoNCE}}(\mU, \mV).$
Suppose $d \ge N-1$, where $d$ is the dimension of embedding vectors $\vu_i^{\star}, \vv_i^{\star}$.
Then,  
$\{\vu_i^{\star} \}_{i=1}^N$ forms a $(N-1)$-simplex ETF and 
$\vu_i^{\star} = \vv_i^{\star}$ holds for all $i \in [N]$.
\end{corollary}

    Note that Corollary~\ref{thm:lu} slightly differs from Theorem 1 of \cite{lu2022neural}, which demonstrates that the optimal embedding vectors for the InfoNCE loss $ \Ls^{\op{InfoNCE}}$ \textit{without any temperature parameter $t$} form a simplex ETF. However, even with the inclusion of a temperature parmeter $t$ in the InfoNCE loss $ \Ls^{\op{InfoNCE}}$, the result remains unchanged. The proof is on Appendix~\ref{appendix:optimal-infonce}.

    We define the antipodal structure as follows. 
    \begin{definition}
    \label{def:antipodal}
    A set of $2N$ vectors $\{\vu_i\}_{i=1}^N$ and $\{\vv_i\}_{i=1}^N$ in $\sR^d$ is called antipodal, if 
    \begin{align}
     \vu_i = -\vv_j, \quad \forall i,j \in [N].
    \end{align}
    \end{definition}
    Note that if vectors $\{\vu_i, \vv_i\}_{i=1}^N$ form the antipodal structure, then the positive pairs are heading to the opposite direction, while the negative pairs are aligned, which is achieving the opposite of what contrastive learning is aiming at.

\subsection{Skeleton of This Paper}

In this paper, we analyze the embedding $\mU^{\star}, \mV^{\star}$ optimized by the sigmoid loss $\Ls^{\op{sig}}$ in Eq.~\ref{loss:siglip}, denoted by 
\begin{align}
    (\mU^{\star}, \mV^{\star}) = \argmin_{\mU, \mV} \Ls^{\op{sig}}(\mU, \mV).
\end{align}

Before solving this problem, in Sec.~\ref{sec:embedding}, we first propose the \textit{double-Constant Embedding Model} (\ours{}), which models the evolution of different embedding vectors using a single parameter $\delta \ge 0$. Note that this model contains various embedding structures including simplex ETF. Interestingly, under some mild assumption (which holds in practice) on the contrastive loss, our proposed model is shown to contain the optimal embedding $(\mU^{\star}, \mV^{\star})$.

Based on this observation, in Sec.~\ref{sec:opt_embed}, we use the proposed \ours{} to explore the optimal embedding for sigmoid loss. Our mathematical results specify the optimal embedding for various hyperparameter settings. Furthermore, our theoretical findings are supported by experimental results on toy datasets. 

\begin{remark}\label{rem:d>=N}
Note that our analysis focuses on the case of $d\geq N$, similar to related works~\citep{lu2022neural,sreenivasan2023mini}.
This is because when $d<N-1$, the ETF solution, which is connected with the optimal solution for our target CL, does not have a standard format for all $N,d$ values~\citep{sustik2007existence,fickus2012steiner, fickus2015tables,azarija2018there}.  
\end{remark}

\section{Double-Constant Embedding Model}\label{sec:embedding}

In this section, we propose a framework that models various types of geometric structures formed by embedding vectors $\{(\vu_i, \vv_i)\}_{i=1}^N$. Since it is not trivial to model arbitrary structures of $2N$ embedding vectors, here we restrict ourselves to model a set of vectors satisfying  
\begin{align}\label{eqn:symmetric_embedding}
 \vu_i^\top \vv_i &= c_1 \quad \text{ for all } i \in [N], \\
 \vu_i^\top \vv_j &= c_2 \quad \text{ for all } i, j \in [N] \text{ with } i \ne j
\end{align}
for some $c_1, c_2 \in [-1, 1]$, \ie 
the matrix $M$ containing elements $M_{ij} = \vu_i^\top \vv_j$ is \textit{double-constant} \citep{o2021double}, where all diagonal elements are $c_1$ and all off-diagonal elements are $c_2$.
Interestingly, as shown later in Corollary~\ref{thm:opt_is_symmetric}, the optimal embedding $(\mU^{\star}, \mV^{\star})$ minimizing the sigmoid loss satisfies the double-constant property in Eq.~\ref{eqn:symmetric_embedding}, which shows that we can safely search the embedding vectors represented by our framework, without losing the optimality.
With this in mind, we now introduce the \textit{double-Constant Embedding Model} (\ours{}).







\begin{definition}[Double-Constant Embedding Model]
\label{def:model}
Suppose $d\geq N$.
Let $\{\vw_i\}_{i=1}^N$ be the $(N-1)$-simplex ETF in $\mathbb{R}^{d-1}$. 
The double-Constant Embedding Model (\ours{}) constructs $2N$ vectors  $(\{\vu_i^{\delta}, \vv_i^{\delta}\}_{i=1}^N)$ as 
\begin{align}
    \vu_i^{\delta} :=
    \begin{pmatrix} 
        \frac{1}{\sqrt{1+\delta^2}} \vw_i  \\
        \frac{1}{\sqrt{1+\delta^2}}\delta 
    \end{pmatrix} 
    \in \mathbb{R}^{d}
    , \quad 
    \vv_i^{\delta} :=
    \begin{pmatrix} 
        \frac{1}{\sqrt{1+\delta^2}} \vw_i  \\
        -\frac{1}{\sqrt{1+\delta^2}}\delta 
    \end{pmatrix} 
    \in \mathbb{R}^{d},
\end{align}
which are parameterized by $\delta \ge 0$. The corresponding embedding matrix is denoted by 
\begin{align}
    \mU (\delta) = [\vu_1^{\delta}, \cdots, \vu_N^{\delta}], \quad \mV (\delta) = [\vv_1^{\delta}, \cdots, \vv_N^{\delta}].
\end{align}
\end{definition}

\subsection{Properties and Examples of \ours{}}

Our proposed model has properties summarized below:
\begin{proposition}\label{prop:model}
Let $\{\vu_i^{\delta}, \vv_i^{\delta}\}_{i=1}^N$ be $2N$ embedding vectors formed by the \ours{} in Def.~\ref{def:model}, for a given $\delta \ge 0$. The inner products between embedding vectors are
\begin{align}\label{eqn:double-const-DEM}
&(\vu_i^{\delta})^\top \vv_i^{\delta}  = \frac{1-\delta^2}{1+\delta^2} \quad \forall i \in [N],
\\
&(\vu_i^{\delta})^\top \vv_j^{\delta}  = -\frac{\frac{1}{N-1}+\delta^2}{1+\delta^2}
\quad \forall i, j \in [N] \text{ with } i \neq j,
\end{align}
i.e., the double-constant property in Eq.~\ref{eqn:symmetric_embedding} is satisfied.
When $\delta=0$, the embedding vectors satisfy
\begin{align}
 \{&\vu_i^{0}\}_{i=1}^N \text{ form } (N-1)\text{-simplex ETF,} \\ 
 &\vu_i^{0} = \vv_i^{0} \quad \forall i \in [N]. 
\end{align}
When $\delta = \infty$, the embedding vectors satisfy
\begin{align}
    &\vu_i^{\infty} = -\vv_j^{\infty} \quad \forall i,j \in [N]
\end{align}
which is ‘antipodal’ embedding defined in Def.~\ref{def:antipodal}.
\end{proposition}

Below we provide an example of embeddings formed by our model, when $d=3$ and $N=3$.
\begin{example}
Fig.~\ref{fig:model} illustrates an example of the \ours{} proposed in Def.~\ref{def:model}, where three embedding vector pairs ($N=3$) are in three dimensional space ($d=3$).
When $\delta=0$, the embedding structure model follows precisely the simplex ETF, which coincides with the result of Prop.~\ref{prop:model}. As $\delta$ increases, the angle between positive pairs $(\vu_i, \vv_i)$ increases for all $i \in [N]$, and the distance between negative pairs $(\vu_i, \vv_j)$ also increases for all $i \ne j$. This means that increasing $\delta$ is pushing both positive pairs and negative pairs apart.
In the extreme case of $\delta=\infty$, we have $\vu_i = - \vv_j$ for all $i,j$, meaning that pairs with positive labels ($\vu_i$ and $\vv_i$) are ignored and pairs with negative labels ($\vu_i$ and $\vv_j$ $\forall i \ne j$) dominate the behavior. Thus, we call this case as the `antipodal' structure.

\end{example}

\begin{figure}[t!]
    \centering
    \begin{tabular}{ccc}
    \includegraphics[height=2cm]{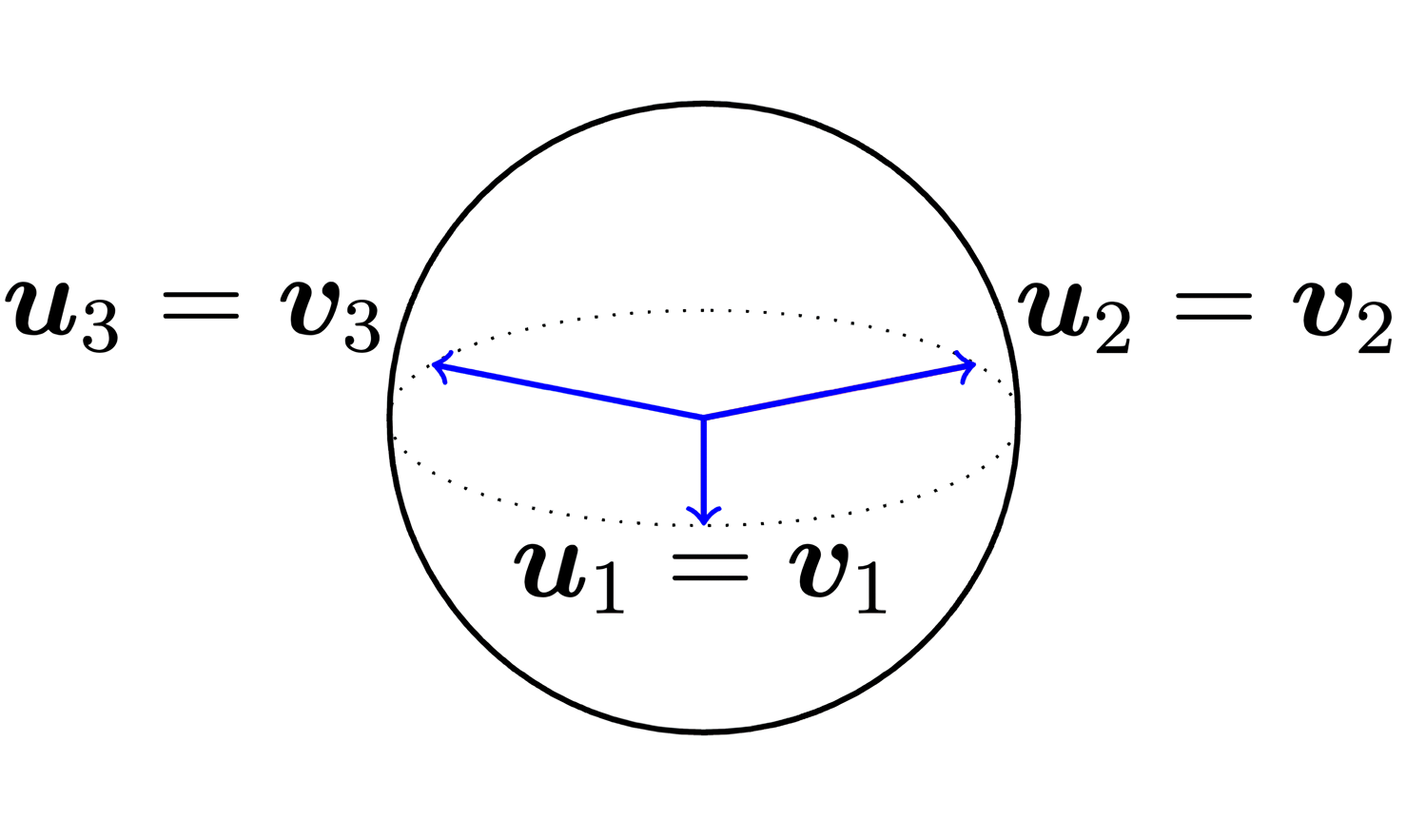} &
    \includegraphics[height=2cm]{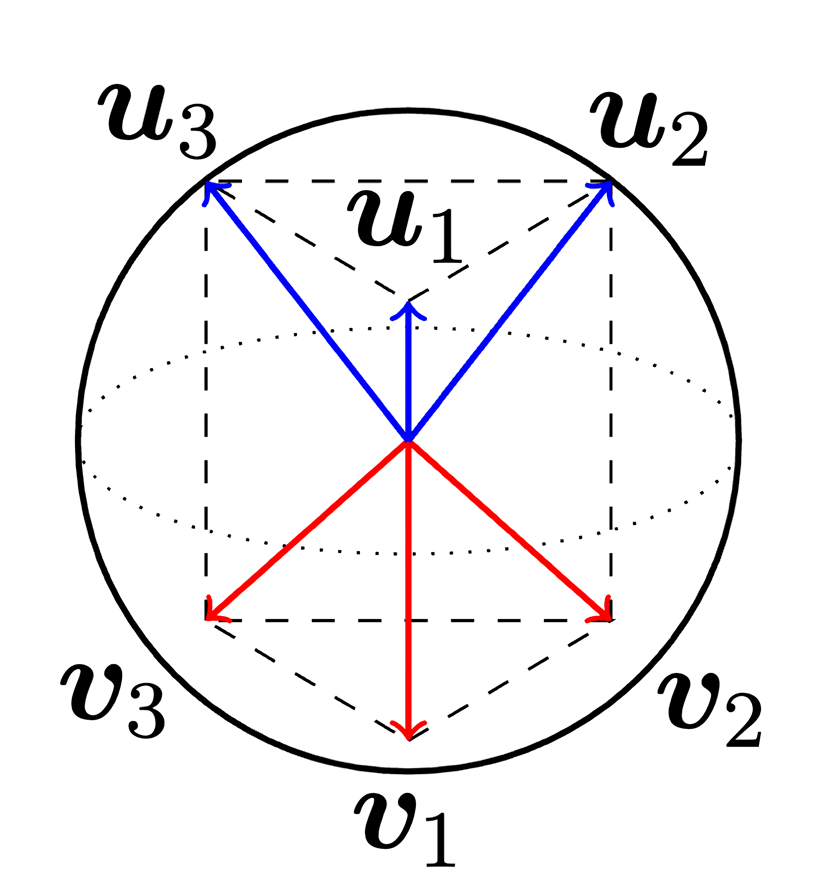} & 
    \includegraphics[height=2cm]{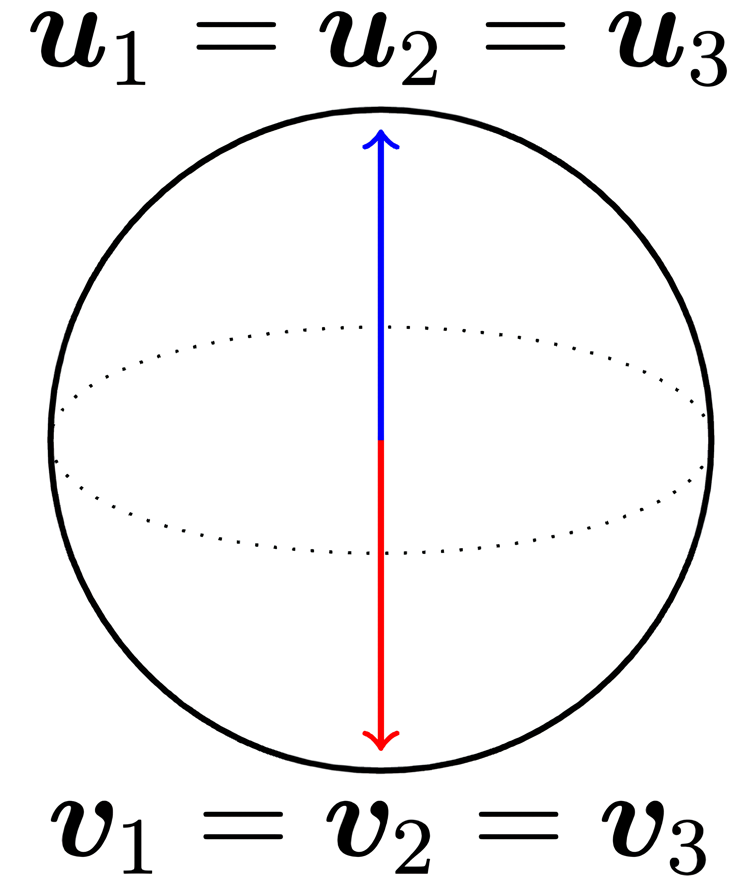} \\
     (a) $\delta=0$ & (b) $\delta=1$ & (c) $\delta=\infty$ \\
    \end{tabular}
    \caption{The proposed double-constant embedding model in Def.~\ref{def:model} for different $\delta$ values, when we have $N=3$ embedding vector pairs $\{(\vu_i^{\delta}, \vv_i^{\delta})\}_{i=1}^N$ in $d=3$ dimensional space. 
    (1) $\delta=0$ (simplex ETF), the angle between positive pair ($\vu_i^{0}$ and $\vv_i^{0}$) is 0.
    (2) $\delta=1$, the angle between $\vu_i^{1}$ and $\vv_i^{1}$ is $\frac{\pi}{2}$.
    (3) $\delta = \infty$, the angle between $\vu_i^{\infty}$ and $\vv_i^{\infty}$ is $\pi$, forming `antipodal' embedding.}
    \label{fig:model}
\end{figure}



\subsection{Optimal Embeddings are in \ours{}}\label{subsec:opt_emb_DEM}

Now we provide theoretical results supporting that the \ours{} proposed in Def.~\ref{def:model} is sufficient to find the optimal embedding that minimizes the sigmoid loss. We start with a theorem stating that the optimal embedding of CL can be found in the \ours{}, under a mild condition on the loss function.

\begin{theorem}\label{thm:DEM is sufficient}
    Consider an arbitrary contrastive loss that follows the form of \begin{align}\label{eqn:loss_symm}
        \Ls(\mU, \mV) = \sum_{i=1}^N  l_{\op{p}}(\vu_i^\top \vv_i) + \sum_{i=1}^N \sum_{j\in [N] \setminus \{i\}} l_{\op{n}}(\vu_i^\top \vv_j),
    \end{align}
where $l_{\op{p}}$ is an arbitrary convex decreasing function for positive pairs, and $l_{\op{n}}$ is an arbitrary convex increasing function for negative pairs.

Then, the optimal embeddings 
$(\mU^{\star}, \mV^{\star}) = \argmin_{\mU, \mV} \Ls(\mU, \mV)   $
can be uniquely found in the \ours{}, \ie 
\begin{align}
    \exists! \delta^\star \ge 0  \text{ such that } \left( \mU(\delta^\star), \mV(\delta^\star)\right) = \left( \mU^{\star}, \mV^{\star}\right).
\end{align}

\end{theorem}

Note that in the above theorem, we assume the loss function $\Ls$ satisfies Eq.~\ref{eqn:loss_symm}. This is quite mild assumption, since various loss functions used in practice satisfy the condition. For example, the sigmoid loss $\Ls^{\op{sig}}$ in Eq.~\ref{loss:siglip} is in the form of Eq.~\ref{eqn:loss_symm} (as shown below), thus the above theorem implies the following.

\begin{corollary}\label{thm:opt_is_symmetric}
Consider the sigmoid loss $\Ls^{\op{sig}}(\mU, \mV)$ defined in Eq.~\ref{loss:siglip}. 
Then,
\begin{align}
\min_{\mU, \mV} \Ls^{\op{sig}}(\mU, \mV) = \min_{\delta \geq 0 } \Ls^{\op{sig}}(\mU(\delta), \mV(\delta)).
\end{align}

In other words, the optimal embedding vectors for the sigmoid loss can be found in the \ours{}.
\end{corollary}
\begin{proof}
The sigmoid loss $\Ls^{\op{sig}}$ in Eq.~\ref{loss:siglip} can be reformulated as Eq.~\ref{eqn:loss_symm}, when $l_{\op{p}} (x) = \log (1 + \exp(-tx + b))$ and $ l_{\op{n}} (x) = \log (1 + \exp(tx - b))$ for a given temperature $t>0$ and bias $b\geq0$. Combining this with Theorem~\ref{thm:DEM is sufficient} completes the proof.
\end{proof}

Now we provide a proof of Theorem~\ref{thm:DEM is sufficient}.
In order to establish Theorem~\ref{thm:DEM is sufficient}, we introduce the mathematical result, which demonstrates that
the \ours{} defined in Def.~\ref{def:model} is a structure that minimizes the sum of the negative pairs' inner product, for a given sum of the positive pairs' inner product:

\begin{lemma}
    \label{thm:model}
    Suppose $d\geq N$.
    Let $\{\vu_i, \vv_i \}_{i=1}^N $ be $2N$ distinct vectors in $\sR^d$, satisfying $\vu_i^\top\vu_i=\vv_i^\top\vv_i=1$ for all $i\in [N]$.
    Then, for every $c  \in[-1,1]$, 
    there exists unique $\delta \ge 0$ such that 
    \begin{align}
    \label{eqn:minimization}
     \left(\mU(\delta), \mV(\delta)\right) =
     \argmin_{(\mU, \mV) \in S_{(c)} } \sum_{i=1}^N\sum_{j\in [N] \setminus \{i\}} \vu_i^\top \vv_j.
    \end{align}
    where $S_{(c)} = \left\{ (\mU, \mV) : \frac{1}{N} \sum_{i=1}^N \vu_i^\top \vv_i=c \right\}$
\end{lemma}
\begin{proof}
    See Appendix~\ref{appendix:first}.
\end{proof}



It should be pointed out that there could be other embedding vectors that satisfy the minimization of Eq.~\ref{eqn:minimization}. However, the unique existence of $\delta$ in Lemma~\ref{thm:model} implies that there is a unique \ours{} structure satisfying the minimization property of Eq.~\ref{eqn:minimization}. This is crucial as we require the double-constant properties of the \ours{} in Eq.~\ref{eqn:double-const-DEM} to prove Theorem~\ref{thm:DEM is sufficient} as below.

\paragraph{Proof of Theorem~\ref{thm:DEM is sufficient}}

\begin{proof}
    Given two embedding matrices $\mU$ and $\mV$, define $\mM$ with $M_{ij}= \vu_i^\top \vv_j$ as the similarity matrix of two embeddings.
    Define the function $h$ of $c\in[-1,1]$ as
     \begin{align}
     h(c) :=
     \min_{\left\{ (\mU, \mV) : \frac{1}{N} \sum_{i=1}^N \vu_i^\top \vv_i=c \right\}}\sum_{i=1}^N \sum_{j\in [N] \setminus \{i\}} \vu_i^\top \vv_j,
    \end{align}
    which is the minimum summation of the off-diagonal terms of similarity matrix $\mM$, among embedding matrices $\mU, \mV$ satisfying that the sum of diagonal terms is $Nc$.
    
    To simplify the notation, define
     \begin{align}
    l(x, y) :=
    l_{\op{p}}\left(\frac{x}{N} \right) +(N-1) l_{\op{n}}\left(\frac{y}{N(N-1)}\right).
    \end{align}

    By using Jensen's inequality on the loss $\Ls$ in Eq.~\ref{eqn:loss_symm},
    \begin{align}
         \frac{1}{N}\Ls &\geq
         l_{\op{p}}\left(\frac{1}{N}\sum_{i=1}^N \vu_i^\top \vv_i\right) +\frac{1}{N}\sum_{i=1}^N \sum_{j\in [N] \setminus \{i\}} l_{\op{n}}(\vu_i^\top \vv_j)
         \\
         & \ge l_{\op{p}}\left(\frac{1}{N}\sum_{i=1}^N \vu_i^\top \vv_i\right) \\
         &+(N-1) \cdot l_{\op{n}}\left(\frac{1}{N(N-1)} \sum_{i=1}^N \sum_{j \in [N]\setminus \{i\}} \vu_i^\top \vv_j\right) \\
         &= l\left(\sum_{i=1}^N \vu_i^\top \vv_i,\;\; \sum_{i=1}^N \sum_{j\in [N] \setminus \{i\}}\vu_i^\top \vv_j\right)
         \\
         &\geq l\left(\sum_{i=1}^N \vu_i^\top \vv_i,\;\;
         h\left(\frac{1}{N} \sum_{i=1}^N \vu_i^\top \vv_i\right) 
         \right)\label{eqn:loss_lower_bound}
    \end{align}
    where the equality conditions of Jensen's inequalities (1st and 2nd inequalities in the above equation) are
    \begin{align}\label{eqn:jensen_eq_condition}
        \vu_i^\top\vv_i &= c_1 \quad \forall i, \\
        \vu_i^\top\vv_j &= c_2 \quad \forall i, j \in [N] \text{ with } i\neq j.
    \end{align}
    The last inequality holds since $l(x,y)$ is an increasing function with respect to $y$.
For every $c \in[-1,1]$, by Lemma~\ref{thm:model}, there is a unique $\delta = \delta(c) \geq0$ 
    such that
    \begin{align}
     (\mU(\delta), \mV(\delta)) =
     \argmin_{(\mU, \mV) \in S_{(c)}} \sum_{i=1}^N \sum_{j\in [N] \setminus \{i\}} \vu_i^\top \vv_j
    \end{align}
    where $S_{(c)} = \left\{ (\mU, \mV) : \frac{1}{N} \sum_{i=1}^N \vu_i^\top \vv_i=c \right\}$.
    It implies 
    \begin{align}
    \label{eqn:thm1:eq9}
    \sum_{i=1}^N \sum_{j\in [N] \setminus \{i\}}\left(\vu_i^\delta\right)^\top \vv_j^\delta  = h \left( \frac{1}{N} \sum_{i=1}^N \left(\vu_i^\delta\right)^\top \vv_i^\delta \right).
    \end{align}
    Moreover, every \ours{} fulfills the equality conditions given in Eq.~\ref{eqn:jensen_eq_condition} 
    , due to its double-constant properties in Eq.~\ref{eqn:double-const-DEM}. Thus, combining these facts with Eq.~\ref{eqn:thm1:eq9}, the inequalities of Eq.~\ref{eqn:loss_lower_bound} reduces to equality as
    \begin{align}
        \label{eqn:thm1:equality}
        &\Ls (\mU(\delta), \mV(\delta)) \\
        & \quad = N \cdot l\left(\sum_{i=1}^N \left(\vu_i^\delta\right)^\top \vv_i^\delta,
         h\left(\frac{1}{N} \sum_{i=1}^N \left(\vu_i^\delta\right)^\top \vv_i^\delta\right) 
         \right).
    \end{align}

    Since $\delta = \delta(c)$ is a function of $c = \frac{1}{N} \sum_{i=1}^N \vu_i^\top \vv_i$, from Eq.~\ref{eqn:loss_lower_bound}, we have
    \begin{align}
        \Ls
        &\geq N \cdot l\left(\sum_{i=1}^N \vu_i^\top \vv_i,\;
         h\left(\frac{1}{N} \sum_{i=1}^N \vu_i^\top \vv_i\right) 
         \right)
        \\
         &\geq N \min_{\frac{1}{N} \sum_{i=1}^N \vu_i^\top \vv_i } l\left(\sum_{i=1}^N \vu_i^\top \vv_i,\;
         h\left(\frac{1}{N} \sum_{i=1}^N \vu_i^\top \vv_i\right) 
         \right)
         \\
        &= N\min_{\delta \geq 0} l\left(\sum_{i=1}^N \left(\vu_i^\delta\right)^\top \vv_i^\delta,\;
         h\left(\frac{1}{N} \sum_{i=1}^N \left(\vu_i^\delta\right)^\top \vv_i^\delta\right) 
         \right)
        \\
        &= \min_{\delta \geq 0} \Ls(\mU(\delta), \mV(\delta))
    \end{align}
    where the last equality holds from Eq.~\ref{eqn:thm1:equality}.

    Therefore, the optimal embeddings ($\mU(\delta^{\star}), \mV(\delta^{\star})$) found from the \ours{} minimize the contrastive loss over all possible embeddings. 

\end{proof}

\begin{remark}
The proposed \ours{} can be effectively utilized not only for the loss in Eq.~\ref{eqn:loss_symm} format, but also for many other contrastive losses, \eg InfoNCE loss. 
Similar to Corollary~\ref{thm:opt_is_symmetric}, 
it can be shown that 
the optimal embedding vectors for InfoNCE loss can be found in \ours{}, as formally stated below:
\begin{corollary}\label{cor:DEM infoNCE}
Consider InfoNCE loss $\Ls^{\op{InfoNCE}}(\mU, \mV)$ defined in Eq.~\ref{loss:info}.
Then,
\begin{align}
\min_{\mU, \mV} \Ls^{\op{InfoNCE}}(\mU, \mV) = \min_{\delta \geq 0 } \Ls^{\op{InfoNCE}}(\mU(\delta), \mV(\delta)).
\end{align}

\end{corollary}
\begin{proof}
    See Appendix~\ref{appendix:ccem-is-sufficeint-for-infonce}.
\end{proof}

\end{remark}

\section{Embedding Learned by Sigmoid Loss}\label{sec:opt_embed}

In this section, we explore the optimal embedding learned by the sigmoid loss using the suggested \ours{}. In Sec.~\ref{sec:optimal}, we provide theoretical results explaining how the structure of optimal embedding varies depending on the temperature and the bias term used in the sigmoid loss, assuming $d \geq N$. The behavior proven in Sec.~\ref{sec:optimal} is empirically confirmed in Sec.~\ref{sec:experiment-main}, not only for $d \geq N$ case but also for $d < N$ case.

\subsection{Theoretical Results
}\label{sec:optimal}

Recall that Corollary~\ref{thm:opt_is_symmetric} implies that the optimal embedding vectors $\mU^{\star}, \mV^{\star}$ can be found by solving 
\begin{align}
    \min_{(\mU, \mV) \in S} \Ls^{\op{sig}}(\mU, \mV),
\end{align}
which is equivalent to finding the optimal parameter 
\begin{align}
    \delta^{\star} := \argmin_{\delta \ge 0} \Ls^{\op{sig}}(\mU(\delta), \mV(\delta)).
\end{align}
The below theorem provides the behavior of $\delta^{\star}$:

\begin{theorem}[Optimal Embedding for sigmoid loss]
    \label{thm:logsig}
    Consider the case when $t=b$, \ie the temperature and the bias are identical.
    
    \vspace{-3pt}
    When $N=3$, we have $\delta^{\star} =  0$, \ie
    the optimal embedding for sigmoid loss forms a simplex ETF, regardless of $t$ value.
    
    \vspace{-3pt}
    On the other hand, when $N\geq 4$, the optimal $\delta^{\star}$ is a monotonic decreasing function of $t$, satisfying 
    \begin{align}
        \delta^{\star} = 
        \begin{cases}
            0, & \text{ if and only if } t >\frac{N-1}{N}\log\left(N-3\right) \\
            \infty, & \text{ if and only if } t<\frac{1}{2}\log\frac{N-2}{2}.
        \end{cases}
    \end{align}
\end{theorem}
\begin{proof}
    See Appendix~\ref{appendix:temperature-siglip}.
\end{proof}
\begin{remark}
    The above theorem implies that the optimal embeddings form a simplex ETF when the temperature is sufficiently large, while they reduce to the `antipodal' embedding when the temperature is smaller than a threshold. Note that the threshold temperature value is in the order of $\Theta(\log N)$.
\end{remark}
\begin{remark}
    According to the empirical results reported in~\cite{zhai2023sigmoid}, SigLIP enjoys high performance \textit{only} when the temperature parameter is set to a sufficiently large number.
    This coincides with our result in Theorem~\ref{thm:logsig} stating that sufficiently large $t$ is necessary to let the trained embeddings form a simplex ETF with $\delta^{\star} = 0 $, instead of the antipodal structure with $\delta^{\star}  = \infty$. 
\end{remark}
Note that Theorem~\ref{thm:logsig} implies that for a given $N$, the optimal $\delta^{\star}$ (representing the distance between positive pairs) decreases as the temperature parameter $t=b$ increases, leading to a tighter alignment of the positive pairs, \ie $\vu_i^\top \vv_i$ approaches to 1. Now we provide a corollary relating the optimal embeddings obtained by InfoNCE loss and sigmoid loss.



\begin{corollary}
    \label{thm:siglip}
    Consider the case when $t=b$.
    When $N=3$, the optimal solution of minimizing $\Ls^{\op{sig}}$ is equal to that of minimizing $\Ls^{\op{InfoNCE}}$, for any temperature parameter $t$.
    When $N\geq 4$,  both solutions are identical, if and only if $t >\frac{N-1}{N}\log\left(N-3\right)$.
\end{corollary}
\begin{proof}
Recall that the optimal embedding for InfoNCE loss $\Ls^{\op{InfoNCE}}$ always follows the simplex ETF irrespective of the temperature parameter $t$, as in Corollary~\ref{thm:lu}. Combining this with Theorem~\ref{thm:logsig} completes the proof. 
\end{proof}

Therefore, by using an appropriate temperature parameter $t > \frac{N-1}{N}\log\left(N-3\right)$, the simplex ETF solution is attainable by training with the sigmoid loss~\citep{zhai2023sigmoid},
which is shown to be more efficient than training with InfoNCE loss.



\begin{figure}[t!]
    \centering
    \includegraphics[width=\linewidth]{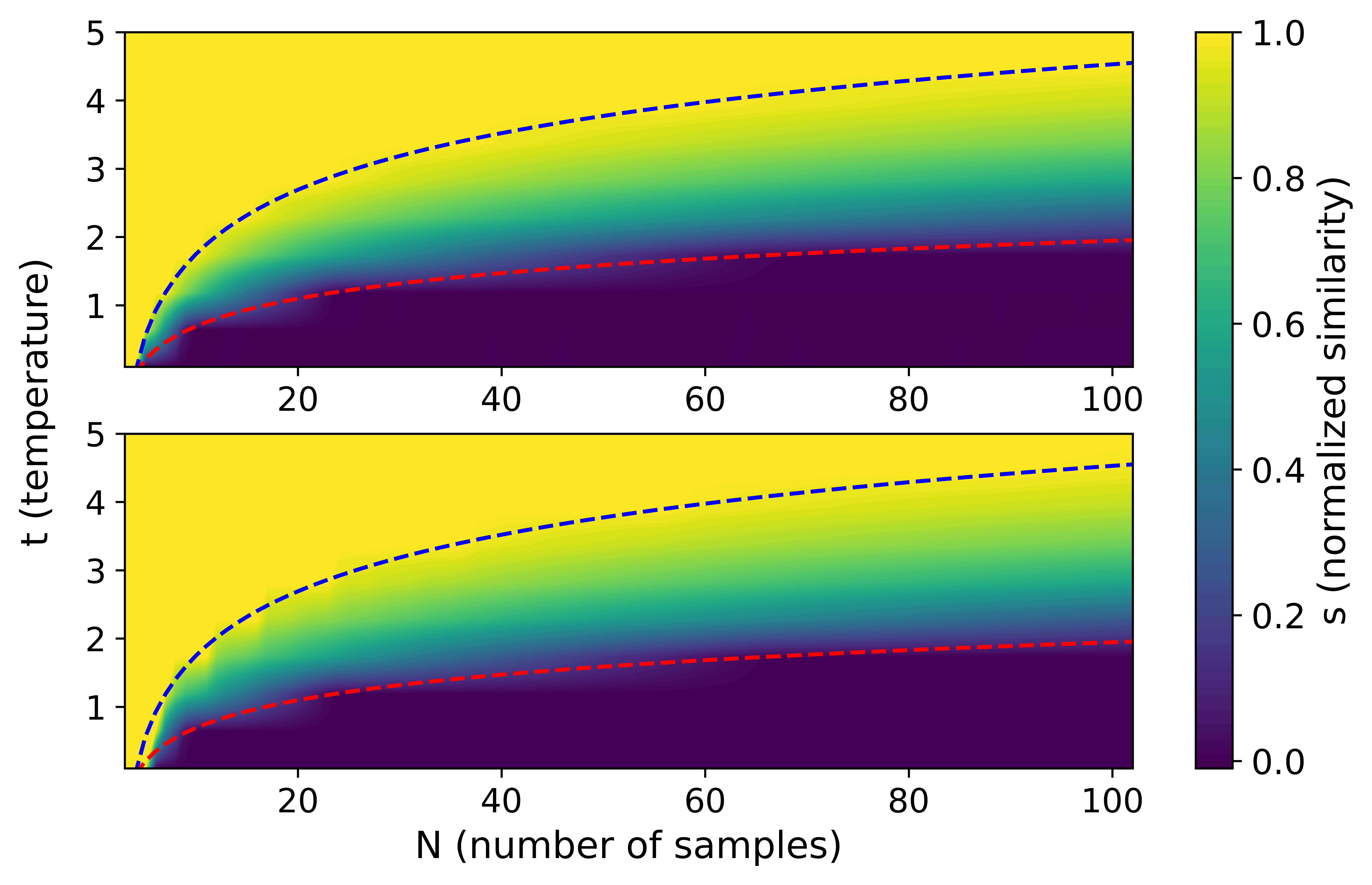}
    \caption{
    The normalized similarity $s$ of positive pairs measured for the embeddings trained by the sigmoid loss $\Ls^{\op{sig}}$, for various $N$ and $t$.
    We test for two cases: $d=N$ (upper) and $d=\frac{N}{2}$ (below).
    The blue dashed line satisfies $t=\frac{N-1}{N}\log(N-3)$, and the red dashed line satisfies $t=\frac{1}{2}\log\frac{N-2}{2}$, which are the threshold values obtained in Theorem~\ref{thm:logsig}. This plot shows that our theoretical results coincide with empirical observations.
    }
    \label{fig:syn_result}
\end{figure}

\subsection{Experimental results\label{sec:experiment-main}}

Now we provide empirical observations supporting our theoretical results. 
Specifically, we test the behavior of the optimal embedding trained by the sigmoid loss for various settings, and check whether it coincides with our result in Theorem~\ref{thm:logsig}.
Although our theoretical results hold only when $d \ge N$, we test whether the behavior of optimal embeddings coincides with the result of Theorem~\ref{thm:logsig} for two cases: 
$d=N$ and $d=\frac{N}{2}$.

\paragraph{Synthetic data without Encoder}

We train embedding vectors $\{\vu_i, \vv_i\}_{i=1}^N$, each of which is initialized by sampling from a standard normal distribution $\mathcal{N}(\mathbf{0}, \mI_d)$, followed by projecting it to a unit sphere. 
Unlike real-world scenarios, we do not use encoder structures; instead, we directly optimize the embedding vectors $\{\vu_i, \vv_i \}_{i=1}^N$ by training them for $50,000$ steps with the sigmoid loss, using learning rate $\gamma=0.5$. 
During the training process, the updated embeddings are projected back to the unit sphere, to make sure that $\vu_i^\top \vu_i = \vv_i^\top \vu_i  = 1$ for all $i \in [N]$. Here, we test various temperature parameters $t=b$, ranging from 0.1 to 5.



Given the trained embeddings $\{\vu_i, \vv_i\}_{i=1}^N$, we measure the normalized similarity of positive pairs, denoted by $ s =  \frac{1}{2} (1 + \frac{1}{N}\sum_{i=1}^N\vu_i^\top\vv_i)$. 
Note that $s=1$ holds when all positive pairs are aligned, while $s=0$ holds when all positive pairs are antipodal.

Fig.~\ref{fig:syn_result} shows the normalized similarity $s$ for various $N$ and $t$. The above plot is the result for $d=N$, while the below plot is for $d=\ceil{N/2}$.
We have two dashed lines visualizing the theoretical threshold values obtained in Theorem~\ref{thm:logsig}. The blue line satisfies $t = \frac{N-1}{N}\log(N-3)$, the minimum value allowing the optimal embedding to be the simplex ETF having $s=1$. 
The red line satisfies $t=\frac{1}{2}\log\frac{N-2}{2}$, the maximum value allowing the optimal embedding to be the antipodal solution having $s=0$.
Thus, our theoretical results claim that the set of $(N,t)$ points above the blue line have $s= 1$, while the points below the red line have $s= 0$, which coincide with the empirical results in Fig.~\ref{fig:syn_result}.

\paragraph{Synthetic data with Encoder}


\begin{figure}[t!]
    \centering
    \includegraphics[width=\linewidth]{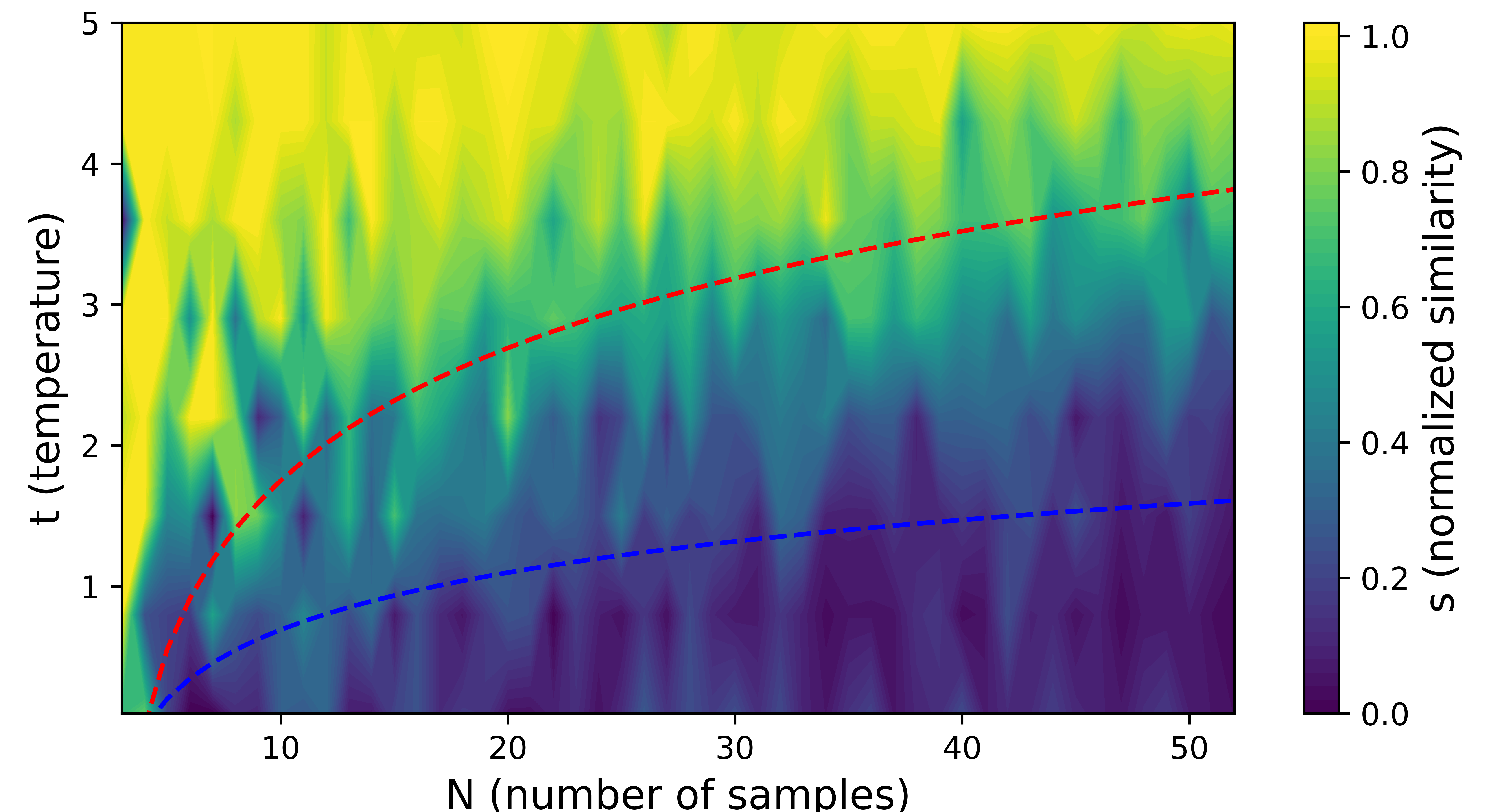}
    \caption{The normalized similarity $s = \frac{1}{2} (1 + \frac{1}{N}\sum_{i=1}^N\vu_i^\top\vv_i)$ of positive pairs measured for the embeddings trained by sigmoid loss $\Ls^{\op{sig}}$, for various $N$ and $t$ when $d=N$. Unlike the setting for Fig~\ref{fig:syn_result}, we train a encoder (two-layer fully-connected ReLU network) which outputs embeddings, rather than directly optimizing the embedding vectors.}
    \label{fig:rebuttal}
\end{figure}

Furthermore, we extend our experiments to more realistic setting where we have an encoder that generates 
the embedding vectors. In contrast to the previous synthetic experiments where the encoder output is directly optimized, we employ a two-layer fully-connected ReLU network denoted as $f$. Especially, $N$ pair of instances $\{\vx_i, \vy_i\}_{i=1}^N$ are generated from a standard normal distribution $\mathcal{N}(\mathbf{0}, \mI_{2d})$, which passes through the encoder $f$ that transforms each $2d$-dimensional vector ($\vx_i$ or $\vy_i$) into a $d$-dimensional embedding vector ($f(\vx_i)$ or $f(\vy_i)$) on a unit sphere. Then, the encoder is trained to minimize the sigmoid loss $\Ls^{\op{sig}}$. Fig.~\ref{fig:rebuttal} shows the similarity of paired  embeddings as a function of the temperature parameter $t$ used in the sigmoid loss, when $d=N$.
This result exhibits a tendency similar to Fig~\ref{fig:syn_result}, implying that the empirical results in a realistic setting also align with the theoretical findings.
\begin{remark}
    Our theoretical results are based on the assumption of $d \ge N$, while our empirical results have similar behaviors for both $d \ge N$ and $d < N$ cases. This empirical finding can be interpreted as follows. Recall that when $d$ is sufficiently large (i.e., $d \ge N$), the learned embeddings satisfy both uniformity and alignment \citep{wang2020understanding}, while the uniformity is not attainable for smaller $d$ (especially, $d < N$). Our empirical results, which measure the alignment of positive pairs, imply that smaller $d$ does not harm the alignment property, although the uniformity is broken. 
\end{remark}

\section{Conclusion}

This paper investigated the embedding structure learned by contrastive learning with the sigmoid loss, motivated by the recent success of SigLIP, a variant of CLIP that uses sigmoid loss instead of InfoNCE loss. 
During the route towards answering the question of finding the optimal embedding for sigmoid loss, we first defined the \textit{double-Constant Embedding Model} (\ours{}) framework that models various embedding structures. Based on our observation that the optimal embeddings are in the \ours{}, we used the \ours{} to specify the optimal embedding structure learned by the sigmoid loss. When the temperature parameter is sufficiently large, the optimal embedding forms a simplex ETF structure, while the optimal embedding reduces to the antipodal structure when the temperature parameter is set to be a small number, which coincides with the empirical observations made by the authors of SigLIP. We also conducted
experiments on synthetic datasets, the results of which coincide with our theoretical results on the behavior of optimal embeddings.

\section*{Acknowledgement}

This research was supported by the MSIT (Ministry of Science and ICT), Korea, under the ICAN (ICT Challenge and Advanced Network of HRD) support program (RS-2023-00259934) supervised by the IITP (Institute for Information \& Communications Technology Planning \& Evaluation).


\clearpage

\bibliography{__ref}

\clearpage
\onecolumn

\appendix

\newtheorem*{theorem*}{Theorem}
\newtheorem*{lemma*}{Lemma}
\newtheorem*{corollary*}{Corollary}

\section{Comparison of Various Contrastive Loss Functions}\label{sec:compare:loss}

As metioned in Sec.~\ref{sec:intro}, there have been various approaches of defining contrastive loss function in the past year. In this section, we compare and highlight the distinctive points of the three loss functions relevant to our paper; NCE loss $\Ls^{\op{NCE}}$~\citep{gutmann2010noise, gutmann2012noise},  InfoNCE loss $\Ls^{\op{InfoNCE}}$~\citep{oord2018representation, radford2021clip}, and Sigmoid loss $\Ls^{\op{sig}}$~\citep{zhai2023sigmoid}.

To simplify the notation, we omit the instance pair index $i$, the temperature parameter, and the bias term. 
For a given embedding $\vu$, we consider the 
the positive embedding $\vv_+$ and the negative embeddings $\{\vv_-^{(j)}\}_{j=1}^N$ with respect to $\vu$. If there is only one negative embedding, we remove the index $(j)$ and denote it as $\vv_-$. 
Then, the three loss functions are defined as follows.
\begin{align}
    \Ls^{\op{NCE}}(\vu, \mV)
    &= - 
    \log\left( \frac{\exp (\vu^\top \vv_+)}{\exp (\vu^\top \vv_+)+\exp (\vu^\top \vv_-)} \right)
    \\ &= 
    - \log\left( \frac{1}{1+\exp \left( - \vu^\top \vv_+ + \vu^\top \vv_-\right)} \right)
    \\
    \Ls^{\op{InfoNCE}}(\vu, \mV)
    &= - 
    \log\left( \frac{\exp (\vu^\top \vv_+)}{\exp (\vu^\top \vv_+)+ \sum_{j}\exp \left(\vu^\top \vv_-^{(j)}\right)} \right)
    \\ &= - 
    \log\left( \frac{1}{1+ \sum_{j}\exp \left( -\vu^\top \vv_+ + \vu^\top \vv_-^{(j)}\right)} \right)
    \\
    \Ls^{\op{sig}}(\vu, \mV)
    &= 
    - \log\left( \frac{1}{1+\exp \left(- \vu^\top \vv_+\right)} \right)
    - \sum_{j} \log\left( \frac{1}{1+\exp \left(\vu^\top \vv_-^{(j)} \right)} \right)
\end{align}
There are two points to note. First, $\Ls^{\op{NCE}}$ has only one negative pair, whereas $\Ls^{\op{InfoNCE}}$, an extension of $\Ls^{\op{NCE}}$, includes multiple negative pairs, similar to $\Ls^{\op{sig}}$. Second, the NCE-based loss functions, $\Ls^{\op{NCE}}$ and $\Ls^{\op{InfoNCE}}$, are derived from the softmax function, while $\Ls^{\op{sig}}$ is formulated based on the sigmoid function. As a result, $\Ls^{\op{sig}}$ can separate the term for the positive pair and negative pairs through summation, whereas the NCE-based losses cannot.

\section{Proofs}\label{sec:proof}

Note that the $(N-1)$-regular-simplex is a set of $N$ vectors which makes the geometry of simplex with equal edges. Detailed concepts and properties are on \cite{el2004certain, gale1953inscribing, fickus2018equiangular, sustik2007existence}.

\subsection{Proof of Lemma~\ref{thm:model}}\label{appendix:first}

\begin{lemma*}
    Suppose $d\geq N$.
    Let $\{\vu_i, \vv_i \}_{i=1}^N $ be $2N$ distinct vectors in $\sR^d$, satisfying $\vu_i^\top\vu_i=\vv_i^\top\vv_i=1$ for all $i\in [N]$.
    Then, for every $c  \in[-1,1]$, 
    there exists unique $\delta \ge 0$ such that 
    \begin{align}\label{eqn:opt_embed_constrained}
     \left(\mU(\delta), \mV(\delta)\right) =
     \argmin_{\left\{ (\mU, \mV) : \frac{1}{N} \sum_{i=1}^N \vu_i^\top \vv_i=c \right\}} \sum_{i=1}^N\sum_{j\in [N] \setminus \{i\}} \vu_i^\top \vv_j.
    \end{align}
\end{lemma*}
\begin{proof}
    By using Jensen’s inequality,
    \begin{align}
        \sum_{i=1}^N \| \vu_i - \vv_i \|_2^2
        &\geq
        N \left\| \frac{1}{N} \sum_{i=1}^N (\vu_i - \vv_i) \right\|_2^2
        \\
        &=
        N\left\| \frac{1}{N}\sum_{i=1}^N (\vu_i + \vv_i) \right\|_2^2
        - 4N\left(\frac{1}{N}\sum_{i=1}^N \vu_i\right)^\top \left(\frac{1}{N}\sum_{i=1}^N\vv_i\right)
        \\
        &\geq
        - 4N\left(\frac{1}{N}\sum_{i=1}^N \vu_i\right)^\top \left(\frac{1}{N}\sum_{i=1}^N\vv_i\right)
    \end{align}
    where the equality conditions are 
    \begin{align}\label{eqn:eq_condition}
    \begin{cases}
        \vu_i -\vv_i =\vc' \text{ for all } i \in [N], \text{ for some constant vector } \vc' \\
        \frac{1}{N}\sum_{i=1}^N \vu_i + \frac{1}{N}\sum_{i=1}^N \vv_i=0 &.
    \end{cases}        
    \end{align}
    This implies that 
    \begin{equation}
        \label{eqn:inequality of inner product}
        \left(\frac{1}{N}\sum_{i=1}^N \vu_i\right)^\top \left(\frac{1}{N}\sum_{i=1}^N\vv_i\right) \geq - \frac{1}{4N} \sum_{i=1}^N \| \vu_i - \vv_i \|_2^2.
    \end{equation}

    Note that the inner product of centroids is
    \begin{align}
        \left(\frac{1}{N}\sum_{i=1}^N \vu_i\right)^\top \left(\frac{1}{N}\sum_{j=1}^N \vv_j\right)
        &=
        \frac{1}{N^2}\sum_{i=1}^N \vu_i^\top \vv_i
        + \frac{1}{N^2}\sum_{i=1}^N \sum_{j\in [N] \setminus \{i\}} \vu_i^\top \vv_j.
    \end{align}

    Since $\| \vu_i - \vv_i \|_2^2=2-2\vu_i^\top \vv_i$ for all $i\in [N]$, we have
    \begin{align}
        \frac{1}{N^2}\sum_{i=1}^N \sum_{j\in [N] \setminus \{i\}} \vu_i^\top \vv_j
        &=
        \left(\frac{1}{N}\sum_{i=1}^N \vu_i\right)^\top \left(\frac{1}{N}\sum_{j=1}^N \vv_j\right)
        - \frac{1}{N^2}\sum_{i=1}^N \vu_i^\top \vv_i
        \\
        &=
        \left(\frac{1}{N}\sum_{i=1}^N \vu_i\right)^\top \left(\frac{1}{N}\sum_{j=1}^N \vv_j\right)
        + \frac{1}{2N^2}\sum_{i=1}^N  \| \vu_i - \vv_i \|_2^2
        - \frac{1}{N}
        \\
        &\geq
        - \frac{1}{4N} \sum_{i=1}^N \| \vu_i - \vv_i \|_2^2
        + \frac{1}{2N^2}\sum_{i=1}^N  \| \vu_i - \vv_i \|_2^2
        - \frac{1}{N}
        \label{eqn:equality_condition}\\
        &=
        \frac{N-2}{2N^2}\sum_{i=1}^N  \vu_i^\top\vv_i
        - \frac{1}{2}
        \label{eqn:minimum:lemma1}
    \end{align}
    where the inequality is from Eq.~\ref{eqn:inequality of inner product}.

    Recall that our target problem in Eq.~\ref{eqn:opt_embed_constrained} is about finding the optimal embedding matrices $(\mU, \mV)$ that minimizes $\sum_{i=1}^N \sum_{j\in [N] \setminus \{i\}} \vu_i^\top \vv_j$, under the constraint of $\frac{1}{N} \sum_{i=1}^N \vu_i^\top \vv_i=c$. One can confirm that the necessary and sufficient condition for achieving this optimal $(\mU, \mV)$ is nothing but 
    the equality condition for Eq.~\ref{eqn:equality_condition}, which is in Eq.~\ref{eqn:eq_condition}.
    
    Note that the construction of the \ours{} in Def.~\ref{def:model} fulfills these equality conditions of Eq.~\ref{eqn:eq_condition}, which also attains the minimum of Eq.~\ref{eqn:minimum:lemma1}, since
        \begin{align}
        \frac{1}{N^2}\sum_{i=1}^N \sum_{j\in [N] \setminus \{i\}} \left(\vu_i^{\delta}\right)^\top \vv_j^{\delta}
        &=
        - \frac{\frac{1}{N} + \frac{N-1}{N} \delta^2}{ 1+\delta^2}
        \\
        &=
        \frac{N-2}{2N^2}\sum_{i=1}^N  \left(\vu_i^{\delta}\right)^\top \vv_i^{\delta}
        - \frac{1}{2}.
    \end{align}

    holds from Prop.~\ref{prop:model}.

    Thus, we can conclude that there exists unique $\delta$ such that $(\mU(\delta), \mV(\delta))$ is the optimal embedding. Actually, we can specify such $\delta$: from Eq.~\ref{eqn:double-const-DEM}, $\delta=\sqrt{\frac{1-c}{1+c}}$ holds when $\frac{1}{N}\sum_{i=1}^N \vu_i^\top \vv_i  = c$. The uniqueness of $\delta$ arises from the fact that $\sqrt{\frac{1-c}{1+c}}$ is a strictly decreasing function with respect to $c\in[-1,1]$.

\end{proof}

\subsection{Proof of Corollary~\ref{cor:DEM infoNCE}}\label{appendix:ccem-is-sufficeint-for-infonce}

\begin{corollary*}
Consider InfoNCE loss $\Ls^{\op{InfoNCE}}(\mU, \mV)$ defined in Eq.~\ref{loss:info}.
Then,
\begin{align}
\min_{\mU, \mV} \Ls^{\op{InfoNCE}}(\mU, \mV) = \min_{\delta \geq 0 } \Ls^{\op{InfoNCE}}(\mU(\delta), \mV(\delta)).
\end{align}
\end{corollary*}
\begin{proof}
    Define the function $h$ of $c\in[-1,1]$ as
     \begin{align}
     h(c) :=
     \min_{\left\{ (\mU, \mV) : \frac{1}{N} \sum_{i=1}^N \vu_i^\top \vv_i=c \right\}}\sum_{i=1}^N \sum_{j\in [N] \setminus \{i\}} \vu_i^\top \vv_j.
    \end{align}

    To simplify the notation, define the convex function
     \begin{align}
    l(x) = \log\left(1+ (N-1) \exp \left(\frac{1}{tN(N-1)} x\right)\right)
    \end{align}

    for any $N\geq 3$ and $t>0$.

    By using Jensen's inequality on the loss $\Ls^{\op{InfoNCE}}$ in Eq.~\ref{loss:info},
    \begin{align}
       \Ls^{\op{InfoNCE}}
        &= \frac{1}{N} \sum_{i=1}^N \log \left( 1+ \sum_{j\in [N] \setminus \{i\}} \exp \left(\vu_i^\top (\vv_j-\vv_i) /t\right)\right)
        + \frac{1}{N} \sum_{i=1}^N \log \left( 1+ \sum_{j\in [N] \setminus \{i\}} \exp \left((\vu_j-\vu_i)^\top \vv_i /t\right)\right)
        \\
        &\geq
        \frac{1}{N} \sum_{i=1}^N \log \left( 1+ (N-1)\exp \left( \frac{1}{N-1}\sum_{j\in [N] \setminus \{i\}}\vu_i^\top (\vv_j-\vv_i) /t\right)\right)
        \\
        & \quad + \frac{1}{N} \sum_{i=1}^N \log \left( 1+ (N-1)\exp \left( \frac{1}{N-1} \sum_{j\in [N] \setminus \{i\}} (\vu_j-\vu_i)^\top \vv_i /t\right)\right)
        \\
        &\geq
        \log \left( 1+ (N-1)\exp \left( \frac{1}{N(N-1)}\sum_{i=1}^N\sum_{j\in [N] \setminus \{i\}}\vu_i^\top (\vv_j-\vv_i) /t\right)\right)
        \\
        & \quad +\log \left( 1+ (N-1)\exp \left(  \frac{1}{N(N-1)}\sum_{i=1}^N \sum_{j\in [N] \setminus \{i\}} (\vu_j-\vu_i)^\top \vv_i /t\right)\right)
        \\
        &=
        2l\left(\sum_{i=1}^N\sum_{j\in [N] \setminus \{i\}} \vu_i^\top \vv_j - (N-1)\sum_{i=1}^N \vu_i^\top\vv_i \right)
        \\
        & \geq
        2l\left( h \left( \frac{1}{N} \sum_{i=1}^N \vu_i^\top \vv_i \right) - (N-1)\sum_{i=1}^N \vu_i^\top\vv_i \right)
    \end{align}
    where the equality conditions of Jensen's inequalities are $\vu_i^\top (\vv_j-\vv_i) = c_1$ and $(\vu_j-\vu_i)^\top \vv_i = c_2$ for all $i, j \in [N]$ with $i\neq j$. The last inequality holds since $l$ is an increasing function.

For every $c \in[-1,1]$, by Lemma~\ref{thm:model}, there is a unique $\delta(c) \geq0$ 
    such that
    \begin{align}
     (\mU(\delta), \mV(\delta)) :=
     \argmin_{\left\{ (\mU, \mV) : \frac{1}{N} \sum_{i=1}^N \vu_i^\top \vv_i=c \right\}} \sum_{i=1}^N \sum_{j\in [N] \setminus \{i\}} \vu_i^\top \vv_j
    \end{align}

which implies 
    \begin{align}
    \sum_{i=1}^N \sum_{j\in [N] \setminus \{i\}}\left(\vu_i^\delta\right)^\top\vv_j^\delta  = h \left( \frac{1}{N} \sum_{i=1}^N \left(\vu_i^\delta\right)^\top \vv_i^\delta \right)
    \end{align}
    Moreover, every \ours{} fulfills the equality conditions of Jensen's inequalities, due to its double-constant properties in Eq.~\ref{eqn:double-const-DEM}. Thus,
    \begin{align}
        \Ls^{\op{InfoNCE}} (\mU(\delta), \mV(\delta))
        = 
         2l\left( h \left( \frac{1}{N} \sum_{i=1}^N \left(\vu_i^\delta\right)^\top \vv_i^\delta \right) - (N-1)\sum_{i=1}^N \left(\vu_i^\delta\right)^\top\vv_i^\delta \right)
    \end{align}

    Since $\delta$ is function of $\frac{1}{N} \sum_{i=1}^N \vu_i^\top \vv_i$, we have
    \begin{align}
        \Ls^{\op{InfoNCE}}
        &\geq 2l\left( h \left( \frac{1}{N} \sum_{i=1}^N \vu_i^\top \vv_i \right) - (N-1)\sum_{i=1}^N \vu_i^\top\vv_i \right)
        \\
         &\geq 2 \min_{\frac{1}{N} \sum_{i=1}^N \vu_i^\top \vv_i }
         l\left( h \left( \frac{1}{N} \sum_{i=1}^N \vu_i^\top \vv_i \right) - (N-1)\sum_{i=1}^N \vu_i^\top\vv_i \right)
         \\
        & = 2\min_{\delta \geq 0} 
        l\left( h \left( \frac{1}{N} \sum_{i=1}^N \left(\vu_i^\delta\right)^\top \vv_i^\delta \right) - (N-1)\sum_{i=1}^N \left(\vu_i^\delta\right)^\top\vv_i^\delta \right)
        \\
        &= \min_{\delta \geq 0} \Ls^{\op{InfoNCE}}(\mU(\delta), \mV(\delta)) .
    \end{align}

    Thus, the optimal embeddings ($\mU(\delta^{\star}), \mV(\delta^{\star})$) found from \ours{} minimize the contrastive loss over all possible embeddings. 
\end{proof}

\subsection{Proof of Corollary~\ref{thm:lu}, using the proposed \ours{}}\label{appendix:optimal-infonce}

Here we show that Corollary~\ref{thm:lu} proven by~\cite{lu2022neural}, can be also proved using the \ours{} (Double-constant Embedding Model) framework proposed by us in Def.~\ref{def:model}.


\begin{corollary*}[Equivalent to Theorem 1 of \cite{lu2022neural}]
    Let $\mU, \mV$ be the optimal embedding minimizing the InfoNCE loss $\Ls^{\op{InfoNCE}}$.
    Then $\mU = \mV$ holds and $\mU$ forms a simplex ETF, regardless of the temperature parameter $t > 0$ and the number of paired instances $N\geq 3$.
\end{corollary*}

\begin{proof}
It is sufficient to use the \ours{} in Def.~\ref{def:model} to find the minimum by Corollary~\ref{cor:DEM infoNCE}. Then, by introducing the \ours{} to the InfoNCE loss  $\Ls^{\op{InfoNCE}}$ of Eq.~\ref{loss:info},

\begin{align}
\Ls^{\op{InfoNCE}}(\mU(\delta), \mV(\delta)) 
&= -\frac{1}{N} \sum_{i=1}^N \log \frac{\exp\left(\left(\vu_i^\delta\right)^\top \vv_i^\delta/t\right)}{\sum_{j=1}^N \exp \left(\left(\vu_i^\delta\right)^\top \vv_j^\delta/t\right)} 
- \frac{1}{N} \sum_{i=1}^N \log \frac{\exp\left(\left(\vu_i^\delta\right)^\top \vv_i^\delta/t\right)}{\sum_{j=1}^N \exp\left(\left(\vu_j^\delta\right)^\top \vv_i^\delta/t\right)}
\\
&= 2
\log \left(
    1
    + (N-1) \exp \left(-\frac{\frac{N}{N-1}}{1+\delta^2}/t \right)
\right)
\end{align}
which is strictly monotonic increasing function with respect to $\delta$ for all $t>0$ and $N \geq 3$.

Therefore, $\Ls^{\op{InfoNCE}}$ minimize at $\delta=0$ for all $t>0$ and $N \geq 3$.
\end{proof}

\subsection{Proof of Theorem~\ref{thm:logsig}}\label{appendix:temperature-siglip}
\begin{theorem*}[Optimal Embedding for sigmoid loss]
    Consider the case when $t=b$, \ie the temperature and the bias are identical.
    When $N=3$, we have $\delta^{\star} =  0$, \ie
    the optimal embedding for sigmoid loss forms a simplex ETF, regardless of $t$ value. On the other hand, when $N\geq 4$, the optimal $\delta^{\star}$ is a monotonic decreasing function of $t$, satisfying 
    \begin{align}
        \delta^{\star} = 
        \begin{cases}
            0, & \text{ if and only if } t >\frac{N-1}{N}\log\left(N-3\right) \\
            \infty, & \text{ if and only if } t<\frac{1}{2}\log\frac{N-2}{2}.
        \end{cases}
    \end{align}
\end{theorem*}
\begin{proof}
It is sufficient to use the \ours{} in Def.~\ref{def:model} to find the minimum by Corollary~\ref{thm:opt_is_symmetric}. Then, by introducing the \ours{} of Def.~\ref{def:model} to the sigmoid loss  $\Ls^{\op{sig}}$ of Eq.~\ref{loss:siglip},

\begin{align}
\Ls^{\op{sig}}(\mU (\delta), \mV (\delta)) 
&= \frac{1}{N} \sum_{i=1}^N  
\log \left(1+\exp\left(-t \left(\vu_i^\delta\right)^\top \vv_i^\delta+b\right)\right)
+ \frac{1}{N} \sum_{i=1}^N \sum_{j\in [N] \setminus \{i\}} 
\log \left(1+\exp\left(t\left(\vu_i^\delta\right)^\top \vv_j^\delta-b\right)\right)
\\
&=
\log \left(1+\exp \left(-t  \frac{1-\delta^2}{1+\delta^2} +b\right)\right)
 + (N-1) \log \left(1+\exp\left(-t\frac{\frac{1}{N-1}+\delta^2}{1+\delta^2}-b\right)\right)
\\
&=
\log \left(1+\exp \left(-t  \frac{-2\delta^2}{1+\delta^2} \right)\right)
 + (N-1) \log \left(1+\exp\left(-t\frac{\frac{N}{N-1}+2\delta^2}{1+\delta^2}\right)\right).
\end{align}

Then,
\begin{align}
\frac{\partial }{\partial \delta} \Ls^{\op{sig}}
&=
\frac{
    \exp \left(-t  \frac{-2\delta^2}{1+\delta^2} \right)
}{
    1+\exp \left(-t  \frac{-2\delta^2}{1+\delta^2} \right)
}
\left(-t  \frac{-4\delta}{(1+\delta^2)^2}\right)
+ 
(N-1) \frac{
    \exp\left(-t\frac{\frac{N}{N-1}+2\delta^2}{1+\delta^2}\right)
}{
    1+\exp\left(-t\frac{\frac{N}{N-1}+2\delta^2}{1+\delta^2}\right)
}
\left(-t  \frac{\frac{2(N-2)}{N-1}\delta}{(1+\delta^2)^2}\right)
\\
&=
\frac{2t\delta}{(1+\delta^2)^2}
\left(
    \frac{
        2
    }{
        1+\exp \left(t  \frac{-2\delta^2}{1+\delta^2} \right)
    }
    - 
    \frac{
        (N-2)
    }{
        1+\exp\left(t\frac{\frac{N}{N-1}+2\delta^2}{1+\delta^2}\right)
    }
\right)
\\
&=
\frac{2t\delta}{(1+\delta^2)^2}
\frac{
    g(N,t,\delta)
}{
    \left(1+\exp \left(t  \frac{-2\delta^2}{1+\delta^2} \right)\right)
    \left(1+\exp\left(t\frac{\frac{N}{N-1}+2\delta^2}{1+\delta^2}\right)\right)
}
\end{align}

where we define $g(N,t,\delta)=(4-N) + 2\exp\left(t\frac{\frac{N}{N-1}+2\delta^2}{1+\delta^2}\right) -(N-2) \exp \left(t  \frac{-2\delta^2}{1+\delta^2} \right)$.

Note that $g(N,t,\delta)$ is strictly monotone increasing with respect to $\delta>0$ for all $t>0$ and $N\geq3$, since 
\begin{align}
    \frac{\partial }{\partial \delta} g(N,t,\delta) 
    &=
    \frac{4t(N-2)\delta}{(1+\delta^2)^2}
    \left(
    \frac{1}{N-1}\exp\left(t\frac{\frac{N}{N-1}+2\delta^2}{1+\delta^2}\right) + \exp \left(t  \frac{-2\delta^2}{1+\delta^2} \right)
    \right)
    \\ &> 0.
\end{align}

Also, note that $\left.\frac{\partial }{\partial \delta}\Ls^{\op{sig}}\right|_{\delta=0} = 0$, $\left.\frac{\partial }{\partial \delta}\Ls^{\op{sig}}\right|_{\delta\to\infty} = 0$, and 
$\left.\frac{\partial^2 }{\partial \delta^2}  \Ls^{\op{sig}} \right|_{\delta=0} 
= 2t \frac{ (3-N) + \exp\left(t\frac{N}{N-1}\right)  }{1+\exp\left(t\frac{N}{N-1}\right)}$.

To find out the condition $t$ of $N$ that minimize $\Ls^{\op{sig}}$, we split the cases for $N=3$ and $N\geq4$.

\paragraph{[ $N=3$ case ]}

First, consider $N=3$. Then, for all $t>0$,
\begin{align}
\left.\frac{\partial^2 }{\partial \delta^2}  \Ls^{\op{sig}} \right|_{\delta=0, N=3} 
&=
2t
\frac{
    \exp\left(t\frac{3}{2}\right)
}{
    1+\exp\left(t\frac{3}{2}\right) 
}
\\ &>0
\end{align}
Furthermore, 
\begin{align}
g(3,t,\delta)
&= 1 + 2\exp\left(t\frac{\frac{3}{2}+2\delta^2}{1+\delta^2}\right) - \exp \left(t  \frac{-2\delta^2}{1+\delta^2} \right)
\\
&= 1 + 2\exp\left(t\left(2-\frac{\frac{1}{2}}{1+\delta^2}\right)\right) - \exp \left(t \left( -2+\frac{2}{1+\delta^2} \right)\right)
\\ &\geq 2\exp\left(\frac{3}{2}t\right)
\\ &> 0
\end{align}

where the inequality holds by plugging $\delta =0$, because it is monotone increasing with respect to $\delta^2$. Since $g$ is non-negative for all $t,\delta>0$, it is also hold that $\left.\frac{\partial}{\partial \delta}  \Ls^{\op{sig}} \right|_{\delta=0,N=3}>0$  for all $\delta>0$.
Therefore, $0=\argmin_{\delta}\Ls^{\op{sig}}(\mU (\delta), \mV (\delta))$ for $N=3$.

\paragraph{[ $N\geq 4$ case ]}

Second, consider the case of $N\geq4$. 

For any $t>\frac{N-1}{N}\log\left(N-3\right)$, $\left.\frac{\partial^2 }{\partial \delta^2}\Ls^{\op{sig}}\right|_{\delta=0} > 0$ holds and
\begin{align}
g(N,t,\delta)
& \geq
g(N,t,0)
\\
&= 2(3-N)+2\exp\left(t\frac{N}{N-1}\right)
\\
& > 0
\end{align}
where the first inequality comes from the fact that $g$ is strictly monotone increasing with respect to $\delta \geq 0$. Since $g$ is non-negative for above case, it is also hold that $\frac{\partial}{\partial \delta}  \Ls^{\op{sig}}>0$  for all $\delta>0$. As a result, $\Ls^{\op{sig}}$ minimize at $\delta=0$.

On the other hand, for $0<t<\frac{1}{2} \log \frac{N-2}{2}$,
\begin{align}
g(N,t,\delta)
& \leq
g(N,t,\infty)
\\
&= (4-N) +2\exp\left(2t\right) - (N-2)\exp\left(-2t\right)
\\
& < 0
\end{align}
holds. Therefore, $\frac{\partial}{\partial \delta}  \Ls^{\op{sig}} <0$  for all $\delta>0$, which implies $\Ls^{\op{sig}}$ minimize at $\delta=\infty$.

Finally, for  $\frac{1}{2} \log \frac{N-2}{2}<t<\frac{N-1}{N}\log\left(N-3\right)$, by applying the inequality opposite to the above, we can easily get $g(N,t,0)<0$ and $g(N,t,\infty)>0$. Since $g$ is strictly increasing continuous function with respect to $\delta$, there is unique $\delta^\star \in (0, \infty)$ such that $g(N,t,\delta^\star)=0$ by the intermediate value theorem. Therefore, there is unique $\delta^\star \in (0, \infty)$ such that minimize $\Ls^{\op{sig}}$.

Furthermore, since $\frac{\partial }{\partial \delta} g(N,t,\delta ) >0$ and $\frac{\partial }{\partial t} g(N,t,\delta ) >0$ holds for all $t, \delta >0$ and $N\geq 3$, by taking partial derivative of $g(N,t,\delta^\star)=0$ for fixed $N$,
\begin{align}
    0&=
    \left(\frac{\partial }{\partial t} g(N,t,\delta^\star) \right)
    +
    \left(\frac{\partial \delta^\star}{\partial t}\right)
    \left(\frac{\partial }{\partial \delta^\star} g(N,t,\delta^\star) \right)
\end{align}
implies $\frac{\partial \delta^\star}{\partial t}<0$ for all $t, \delta>0$ and $N\geq3$. Therefore, the optimal $\delta^{\star}$ is a monotonic decreasing function of $t$ for fixed $N$.
\end{proof}

\clearpage

\end{document}

%% file: math_commands.tex
\usepackage{amsmath,amsfonts,bm}

\def\eqref#1{equation~\ref{#1}}

\def\ceil#1{\lceil #1 \rceil}

\def\1{\bm{1}}

\def\vc{{\bm{c}}}

\def\vu{{\bm{u}}}
\def\vv{{\bm{v}}}
\def\vw{{\bm{w}}}
\def\vx{{\bm{x}}}
\def\vy{{\bm{y}}}

\def\mI{{\bm{I}}}

\def\mM{{\bm{M}}}

\def\mU{{\bm{U}}}
\def\mV{{\bm{V}}}

\DeclareMathAlphabet{\mathsfit}{\encodingdefault}{\sfdefault}{m}{sl}
\SetMathAlphabet{\mathsfit}{bold}{\encodingdefault}{\sfdefault}{bx}{n}

\def\sR{{\mathbb{R}}}

\newcommand{\Ls}{\mathcal{L}}

\DeclareMathOperator*{\argmin}{arg\,min}